\documentclass{article}

\usepackage[preprint,nonatbib]{neurips_2020}

\usepackage[utf8]{inputenc} 
\usepackage[T1]{fontenc}    
\usepackage{hyperref}       
\usepackage{url}            
\usepackage{booktabs}       
\usepackage{amsfonts}       
\usepackage{nicefrac}       
\usepackage{microtype}      
\usepackage{bm}
\usepackage{amsmath, amsthm, amssymb}
\usepackage{graphicx}
\usepackage{subcaption}
\usepackage{hyperref}
\hypersetup{ 
    colorlinks=true,
    citecolor=blue
}

\usepackage{xcolor}

\newtheorem{prop}{Proposition}

\usepackage{newfloat}
\DeclareFloatingEnvironment[
    fileext=los,
    listname={List of appendix figures},
    name=Figure,
    placement=tbhp,
    within=section,
]{apdxfig}

\DeclareFloatingEnvironment[
    fileext=los,
    listname={List of appendix tables},
    name=Table,
    placement=tbhp,
    within=section,
]{apdxtab}

\author{%
 Yan Zhang$^{1}$ \; Michael J. Black$^2$\; Siyu Tang$^{1}$ \\
 $^1$ETH Z\"{u}rich, Switzerland \\
 $^2$Max Planck Institute for Intelligent Systems, T\"{u}bingen, Germany \\
  \texttt{\{\href{mailto:yan.zhang@inf.ethz.ch}{yan.zhang},\href{mailto:siyu.tang@inf.ethz.ch}{siyu.tang}\}@inf.ethz.ch}\\
  \texttt{\{\href{mailto:yan.zhang@tue.mpg.de}{yan.zhang},\href{mailto:black@tuebingen.mpg.de}{black},\href{mailto:stang@tue.mpg.de}{stang}\}@tuebingen.mpg.de}
}

\title{Perpetual Motion: \\ Generating Unbounded Human Motion}

\begin{document}

\maketitle

\begin{center}
  \newcommand{\teaserwidth}{\textwidth}
  \vspace{-0.5cm}
  \centerline{\includegraphics[width=\linewidth]{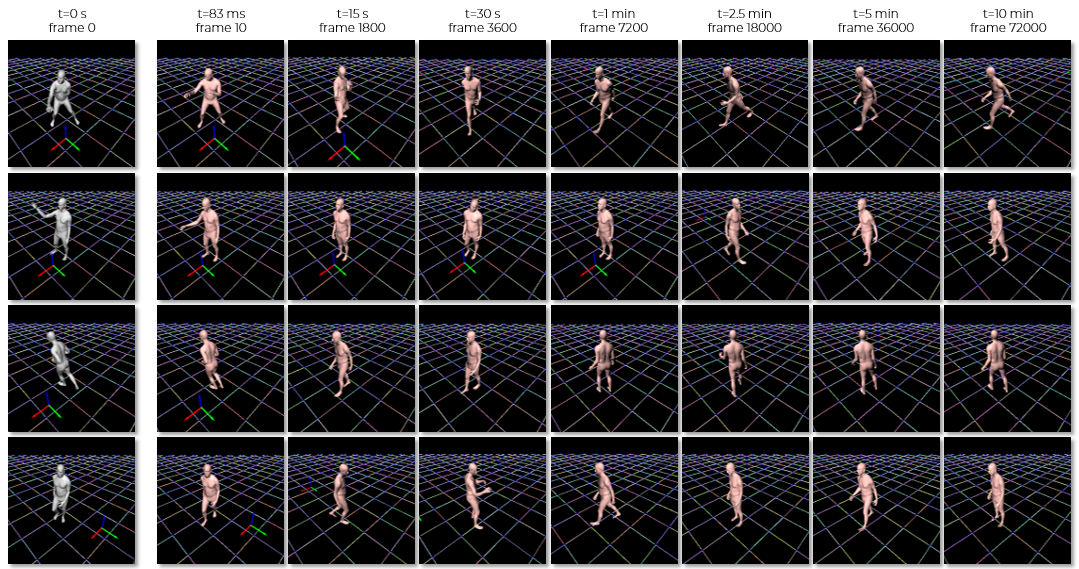}}
  \captionsetup{labelformat=empty}
    \captionof{figure}{Given an initial body configuration, our method generates ``perpetual'' human motion with ever-changing limb poses. After 10 minutes, the body pose still varies, and the motion remains realistic.}
    \label{fig:teaser}
\end{center}%


\begin{abstract}
  The modeling of human motion using machine learning methods has been widely studied. 
 In essence it is a time-series modeling problem involving predicting how a person will move in the future given how they moved in the past.
 Existing methods, however, typically have a short time horizon, predicting a only few frames to a few seconds of human motion.  
 Here we focus on long-term prediction; that is, generating long sequences (potentially infinite) of human motion that is plausible. 
 Furthermore, we do not rely on a long sequence of input motion for conditioning, but rather, can predict how someone will move from as little as a single pose.
 Such a model has many uses in graphics (video games and crowd animation) and vision (as a prior for human motion estimation or for dataset creation).
 To address this problem, we propose a model to generate non-deterministic, \textit{ever-changing}, perpetual human motion, in which the global trajectory and the body pose are cross-conditioned. 
We introduce a novel KL-divergence term with an implicit, unknown, prior.
 We train this using a heavy-tailed function of the KL divergence of a white-noise Gaussian process, allowing latent sequence temporal dependency. 
 We perform systematic experiments to verify its effectiveness and find that it is superior to baseline methods. 
  
\end{abstract}

\section{Introduction}
\label{sec:intro}

Given a static human pose, e.g.~a person is sitting on a sofa, we are able to predict plausible motion sequences of a person over long time horizons, e.g.~the person might stand up and walk out of the room or the person might lie down on the sofa to have a rest. 
Despite a long history of methods that learn to model human motions, generating natural movements over arbitrarily long sequences remains extremely challenging.
There are two major reasons. First, long-term human motion is intrinsically stochastic. Without the ability to model a rich set of future motion sequences, existing work often results in mode collapse, static poses or unrealistic body configurations.  
Second, human motion consists of a global motion trajectory and local variation in the limb poses. 
To be perceived as a natural motion sequence, the global trajectory and the local poses have to correspond in a physically plausible way.

Because of these challenges, most existing learning-based approaches focus on short-term motion prediction, the aim of which is to predict a deterministic human motion sequence in the very near future (often less than a second), based on a relatively long observation (e.g.~\cite{aksan2019structured,martinez2017human}). The issue is so prevalent that synthesizing motions as short as several seconds is considered ``long-term''. Furthermore, when the global motion trajectory is considered, it is often regarded as a pre-defined input from a user or a separate path planning module.
In this work, our goal is to generate significantly longer, or ``perpetual'', motion: given a short motion sequence or even a static body pose, the goal is to generate non-deterministic \textit{ever-changing} human motions in the future. To this end, we design a two-stream variational autoencoder (see Sec.~\ref{sec:method}) with RNNs, in which the change of body pose and the change of the body translation are conditioned on each other. The model is then learned 
from motion capture data \cite{AMASS:2019}, without additional information like user input or action labels. 

Similar to work that synthesizes generic time series, e.g.~\cite{chung2015recurrent,sonderby2016ladder,aksan2018stcn}, our novel model is auto-regressive over time, and has stochastic modules inside. 
During training, an evidence lower bound (ELBO) is maximized. During testing, the latent variables are sampled from the inference posterior. 
Importantly, our model does not have an explicit prior model as in other studies. 
Instead, we apply a Charbonnier penalty function \cite{charbonnier1994two} on the KL-divergence term, and the implicit latent sequence prior is then different from a standard normal distribution, making the latent variables possess temporal dependencies. In addition, our novel KL-divergence term still retains a valid ELBO, and we observe that it effectively overcomes posterior collapse during training.

To verify the effectiveness of our method, we perform systematic experiments to evaluate the model's representation power, analyze the frequency and diversity of the generated human motions, and conduct a perceptual study to evaluate the naturalness of the generated motions. 
We show that the proposed method outperforms two state-of-the-art baseline methods. Qualitatively, after generating 72000 frames (10 minutes) of motion with our method, the body motion is still plausible.

In summary, our contribution is as follows: 
(1) To address the task of generating perpetual motion from short-term sequences or static poses, we propose a two-stream cross-conditional variational RNN network. Its effectiveness and superior performance to state-of-the-art baseline methods are verified by experiments.
(2) We establish a systematic evaluation pipeline to verify the effectiveness of methods, from the perspective of model representation power, motion frequency, diversity, and naturalness. 
(3) We design a novel KL-divergence term to implicitly make the latent sequence prior possess temporal dependency. Also, this novel KL-divergence term still leads to a valid ELBO, and effectively avoids posterior collapse.

\section{Related Work}
\label{sec:related_work}

\paragraph{Short-term motion prediction.}
From the motion prediction perspective, it is expected to produce deterministic and accurate short sequences. In \cite{martinez2017human}, the model is trained over 1 second of motion, and predicts motion up to 400 milliseconds in the future. This experimental setting is widely used in follow-up studies, such as in \cite{Gui_2018_ECCV,pavllo2018quaternet,pavllo2019modeling,ghosh2017learning,gui2018adversarial,li2018convolutional,wang2019imitation,aksan2019structured}. Although long-term prediction is also considered in these studies, the generated future sequence is often in the range of seconds. In addition, the work \cite{Hernandez_2019_ICCV} formulates motion prediction and planning as temporal in-painting problems. Rather than performing evaluation w.r.t. joint position errors, they propose to measure the power spectrum distribution as a metric to compare their generated results with the ground truth.

\paragraph{Long-term motion generation and character animation.}
From the animation perspective, generated motions are much longer. 
The work of \cite{yan2019convolutional} generates the motion sequence as a whole from a pre-defined latent Gaussian process. In their experiments, sequences of 1000 frames (about 1 minutes) are generated. The work \cite{pavllo2018quaternet} also designs a network for character control, which is based on user inputs and given walking paths. The studies of \cite{holden2017phase,starke2019neural} can generate quite long human motion as well, while the motion is blended from several pre-defined action categories. Physical simulation is employed for motion generation in some studies like \cite{peng2016terrain}, which proposes a reinforcement learning method with physical constraints. The mass, size, friction factor and other attributes of the character model are specified in advance.

\paragraph{Deep variational Bayesian methods for generic time series modeling.}
As human motion is a special time series, deep variational Bayesian methods for generating speech, handwriting, natural language, and other types of time series are related as well. 
The work of VRNN \cite{chung2015recurrent} designs a learnable prior to incorporate temporal dependencies of latent variables. To improve the performance, the work of \cite{sonderby2016ladder} proposes a latent prior with ladder structures for temporal dependency modeling. Based on such ladder structure, STCN \cite{aksan2018stcn} employs a temporal convolutions to process information, and achieves state-of-the-art performance for speech generation and handwriting generation.

\paragraph{Ours versus others.} Methodologically, our method is a special type of deep variational Bayesian method, and practically, our method aims at generating endless motions over time. Therefore, in this paper, we compare our method with two state-of-the-art methods from individual fields, which are conservatively modified to fit our task for fair comparison. See Sec.~\ref{sec:baseline_method} for details.

\section{Method}
\label{sec:method}

\subsection{Human Motion Representation}
\label{sec:motion_repr}

In this paper, we represent human motion with time time sequences of the 3D pelvis location, $\bm{T}=\{ {\bm t}_i\}$, and the articulated body pose, $\bm{\Theta}=\{ {\bm \theta}_i\}$. The body is represented by the SMPL \cite{SMPL:2015} kinematic tree rooted at the pelvis and includes 22 body joints. We use the relative joint rotation in this kinematic tree to represent the body pose, with rotations represented in a 6D continuous space \cite{zhou2019continuity}, which has proven useful for back-propagation. Therefore, the body pose at each time instant has ${\bm \theta}_i \in \mathbb{R}^{132}$, as well as the body translation at each time is ${\bm t}_i \in \mathbb{R}^{3}$.
The pelvis location and rotation are with respect to the world coordinate system.

To make the motion representation invariant to the world coordinate system, we unify the world coordinates of different motion sequences. For each sequence, we set the negative gravity direction as the Z-axis, set the horizontal component from the left hip to the right hip as the X-axis, and determine the Y-axis according to the right-hand rule. Also, the world origin is located at the body pelvis in the first frame. As a pre-processing step, transforming every mocap sequence to this new world coordinate is conducted before training models. As a post-processing step, we transform all generated bodies back to their original world coordinates. Since we deal with inconsistent world coordinates in the AMASS dataset \cite{AMASS:2019}, we refer to this new world coordinate setting as ``AMASS coordiante''.

\subsection{Network Design}
We use the variational autoencoder framework \cite{vae} to model a generic time sequence ${\bm X}_{1:i}=\{ {\bm x}_1, {\bm x}_2, ..., {\bm x}_i\}$. As in \cite{bayer2014learning}, we have
\begin{equation}
    \log p({\bm X}_{1:i}) = \log p({\bm x}_1) + \sum_{i} \log p({\bm x}_i | {\bm X}_{1:i-1}).
\end{equation}
For each conditional probability, we derive an evidence lower bounded (ELBO) as 
\begin{equation}
    \label{eq:elbo}
    \begin{split}
    \log p({\bm x}_i | {\bm X}_{1:i-1}) &\geq \mathbb{E}_{\bm{Z}_{2:i} \sim q_{\alpha}(\bm{Z}_{2:i} | {\bm X}_{1:i-1}) }[ \log p_{\beta}({\bm x}_i | \bm{Z}_{2:i}, {\bm X}_{1:i-1}) ]\\ &-D_{KL}\left( q_{\alpha}(\bm{Z}_{2:i} | {\bm X}_{1:i-1})\, || \, p(\bm{Z}_{2:i} | {\bm X}_{1:i-1}) \right),
    \end{split}
\end{equation}
in which $q_{\alpha}$ is the inference model (encoder), $p_{\beta}$ is the generation model (decoder), and $q_{\alpha}(\bm{Z}_{2:i} | \bm{x}_i)=q_{\alpha}(\bm{Z}_{2:i})$ is assumed to enable the encoder to perform prediction.

\paragraph{Cross-conditional two-stream variational RNN.} 
We design an auto-regressive model illustrated in Fig. \ref{fig:crossmotionvar}, in which the translation stream and the body pose stream are conditioned on each other. The activation function $\sigma(\cdot)$ is Swish \cite{ramachandran2017searching}, i.e. $\sigma(x) = x\cdot\text{sigmoid}(x)$, which makes training faster in our trials.
In this network, the encoder takes the body configuration in the current frame as input, and combines it with the RNN states incorporating information from inputs in the past. Therefore, the posterior model $q_{\alpha}(\bm{Z}_{2:i} | {\bm X}_{1:i-1})$ is effectively modeled. In addition, the RNN in the decoder uses the latent sequence up to the current time to produce the output. Combining with the residual connection, the generation model $p_{\beta}({\bm x}_i | \bm{Z}_{2:i}, {\bm X}_{1:i-1})$ is effectively modeled.

\paragraph{$\alpha$-residual connection.}
Previous work has reported that RNNs can lead to first-frame jump artifacts; i.e.~a considerable discontinuity between the last given input frame and the first generated frame. Martinez et al. \cite{martinez2017human} use a residual connection to overcome this artifact. However, the standard residual connection does not work for us, since the training loss is very small in the beginning, and the model parameters are not updated. To overcome this issue, we employ the {\em exponential moving average scheme} as a type of residual connection. With a hyper-parameter $\alpha$, this ``$\alpha$-residual'' gives ${\bm x}_{i} = \alpha {\bm x}_{i-1} + (1-\alpha) f(\bm{X}_{1:i-1})$, in which $f(\cdot)$ incorporates other network modules. Note that this $\alpha$ cannot be learned via back-propagation, since it will increase rapidly towards 1 and the network parameters will not effectively update. By default, the $\alpha$-residual is set to $\alpha=0.9$ in our trials.
\begin{figure}
    \centering
    \includegraphics[width=0.99\linewidth]{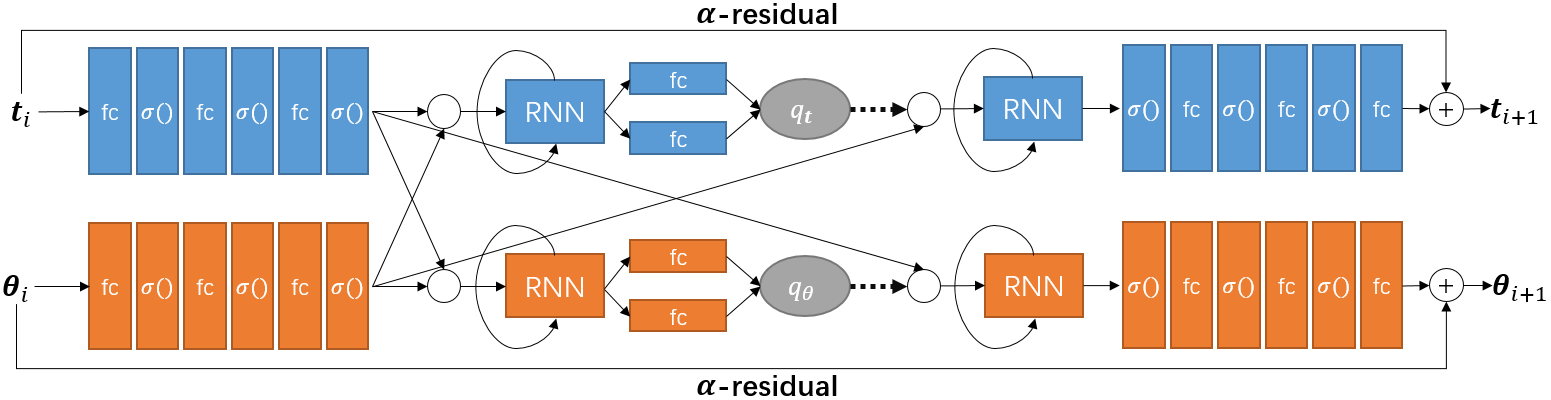}
    \caption{Illustration of our cross-conditional two-stream variational RNN architecture. The blue and orange color denote the body translation stream and the body pose stream, respectively. The circles denote feature concatenation. The dash arrows denote random sampling.  }
    \label{fig:crossmotionvar}
\end{figure}

\subsection{Training Loss}
\label{sec:loss}

Provided ground truth body translations $\bm{T}_{1:i}$ and body poses $\bm{\Theta}_{1:i}$, we feed the sub-sequences $\bm{T}_{1:i-1}$ and $\bm{\Theta}_{1:i-1}$ to the model as input, and obtain $\hat{\bm{T}}_{2:i}$ and $\hat{\bm{\Theta}}_{2:i}$.
The training loss comprises three components, and is given by $\mathcal{L} = \mathcal{L}_{rec} + \lambda_{vp} \mathcal{L}_{vposer} + \lambda_{kl} \mathcal{L}_{KL}$.


\paragraph{The reconstruction loss $\mathcal{L}_{rec}$:} This component corresponds to the expectation term in the ELBO (Eq.~\ref{eq:elbo}), and is given by
\begin{equation}
    \begin{split}
    \mathcal{L}_{rec} &= | \hat{\bm{T}}_{2:i} - \bm{T}_{2:i} | + | \hat{\bm{\Theta}}_{2:i} - \bm{\Theta}_{2:i} | \\
    &+ \lambda_{ts} \left( | (\hat{\bm{T}}_{3:i}-\hat{\bm{T}}_{2:i-1}) - (\bm{T}_{3:i}-\bm{T}_{2:i-1}) | + | (\hat{\bm{\Theta}}_{3:i}-\hat{\bm{\Theta}}_{2:i-1}) - (\bm{\Theta}_{3:i}-\bm{\Theta}_{2:i-1}) | \right),
    \end{split}
    \label{eq:loss_recon}
\end{equation}
which includes frame-wise and time difference reconstruction. $\lambda_{ts}$ is a hyper-parameter.  

\paragraph{The pose naturalness loss $\mathcal{L}_{vp}$:} Since our body is represented in the SMPL kinematic tree, we can use the pre-trained VPoser \cite{pavlakos2019expressive} to encourage naturalness of generated body poses, as employed in \cite{pavlakos2019expressive,PROX:2019,PSI:2019}. Specifically, given the VPoser encoder $\phi(\cdot)$, we have
\begin{equation}
    \mathcal{L}_{vp} = |\phi(\hat{\bm{\Theta}}^p_{2:i})|^2,
\end{equation}
in which $\hat{\bm{\Theta}}^p$ is the predicted body poses excluding the global orientation (the pelvis rotation).

\paragraph{The KL-divergence loss $\mathcal{L}_{kl}$:} 
In our cross-conditional VRNN model, we let $\mathcal{L}_{KL}$ be the average of our proposed KL divergence terms corresponding to the two streams.
Studies like \cite{chung2015recurrent,sonderby2016ladder,aksan2018stcn,razavi2018preventing} explicitly design the latent sequence prior with temporal dependency. 
During testing, latent variables are sampled from the designed prior distribution. 
Despite effectiveness, the prior distribution design might not be appropriate, and extra computational cost is involved, as the latent prior as well as the inference posterior are both learned from data.   
In our work, we do not formulate the latent prior explicitly. Instead, we propose a novel KL-divergence term, which implicitly allows the prior to possess temporal dependency. Thus, the latent prior does not have an explicit formula. During testing, we draw latent samples from the inference posterior $q_{\alpha}(\cdot)$, which is dependent on RNN states and is regularized by that implicit latent prior. Specifically, we assume the latent prior not to be a standard normal distribution, then we have $ D_{KL}\left( q_{\alpha}(\bm{Z}_{2:i} | {\bm X}_{1:i-1})\, || \, p(\bm{Z}_{2:i} | {\bm X}_{1:i-1}) \right) \neq D_{KL}\left(q_{\alpha}(\bm{Z}_{2:i} | {\bm X}_{1:i-1})\, || \, \mathcal{N}(0, I) \right) $. We let 
\begin{equation}
    \label{eq:new_kl}
     D_{KL}\left( q_{\alpha}(\bm{Z}_{2:i} | {\bm X}_{1:i-1})\, || \, p(\bm{Z}_{2:i} | {\bm X}_{1:i-1}) \right) = \Psi \left( D_{KL}\left(q_{\alpha}(\bm{Z}_{2:i} | {\bm X}_{1:i-1})\, || \, \mathcal{N}(0, I) \right) \right),
\end{equation}
in which $\Psi(\cdot)$ is the Charbonnier penalty function, $\Psi(s) = \sqrt{1+s^2}-1$, \cite{charbonnier1994two}. To investigate its properties, without loss of generality, we assume the feature dimension in sequences $\bm{X}$ and $\bm{Z}$ is 1D. We find that
\begin{prop}
    The new KL-divergence in Eq. \ref{eq:new_kl} can:
    (1) lead to a higher ELBO than its counterpart with a standard normal distribution prior, (2) introduce temporal dependencies in the latent space, (3) avoid posterior collapse numerically, and (4) retain a low computational cost.
\end{prop}
\begin{proof}
We factorize 
\begin{equation}
q_{\alpha}(\bm{Z}_{2:i} | {\bm X}_{1:i-1}) = q(z_2 | \bm{X}_{1:1}) \prod_{m=3}^i q(z_m | \bm{Z}_{2:m-1}, \bm{X}_{1:m-1}).   
\end{equation}
According to our network design, we actually have $q(z_m | \bm{Z}_{2:m-1}, \bm{X}_{1:m-1}) = \mathcal{N}(\mu_{q,m}, \sigma_{q,m}^2)$, in which $\mu_{q,m}$ and $\sigma_{q,m}$ are derived based on the RNN states (Fig. \ref{fig:crossmotionvar}). Therefore, we can rewrite the KL-divergence term to
\begin{equation}
    \begin{split}
    &D_{KL}\left( q_{\alpha}(\bm{Z}_{2:i} | {\bm X}_{1:i-1})\, || \, p(\bm{Z}_{2:i} | {\bm X}_{1:i-1}) \right) \\
    &=\sqrt{1+ \left(\sum_m (\sigma_{q,m}^2+\mu_{q,m}^2 -1 -\log \sigma_{q,m}^2) \right)^2 }-1
    = \sqrt{1+ \left(\sum_m \phi_{q,m} \right)^2 }-1 \\
    &= \sqrt{1+ \sum_{m,n} \phi_{q,m}\phi_{q,n}  }-1 \leq \sum_m \phi_{q,m}=D_{KL} (q_{\alpha}(\bm{Z}_{2:i} | {\bm X}_{1:i-1}) || \mathcal{N}(0, \bm{I})).
\end{split}
\end{equation}
with $\phi_{q,m}$ being the KL-divergence between the posterior and the standard normal distribution at time $m$. From the above derivation, we can see that the temporal correlation term $\phi_{q,m}\phi_{q,n}$ appears in the formula. Also, with the same generation model $p_{\beta}(\cdot)$ and the reconstruction loss, our novel KL-divergence term leads to a higher ELBO. 

Since the Charbonnier penalty function is a scalar function, our method retains a low computation cost, unlike methods with an explicit latent prior, e.g. \cite{chung2015recurrent,aksan2018stcn}. In addition, it is noted the derivative of the Charbonnier function is $\Psi^{'}(s) = -s/\sqrt{1+s^2}$. Consequently, gradients for updating the KL divergence term will get small, when the KL-divergence $D_{KL}(q_{\alpha} || \mathcal{N}(0, \bm{I}))$ is small. Numerically, it effectively overcomes the posterior collapse problem.
\end{proof}


In our experiments, the training loss weights are set to $\{\lambda_{ts}, \lambda_{vp}, \lambda_{KL}\}=\{5, 10^{-4}, 1\}$ for all trials, without weight annealing. 
With a trained model, motion is generated in an auto-regressive manner after providing the initial frame(s). Due to the sampling module, the generated motions are not deterministic.

\section{Experiments}
\label{sec:experiments}

\subsection{Datasets}
In our experiments, we use {\bf AMASS} \cite{AMASS:2019}, which unifies diverse motion capture data into the SMPL \cite{SMPL:2015} body representation. We use {\bf ACCAD} and {\bf CMU} for training, and use {\bf HumanEva} and {\bf MPI-Mosh} for testing, which are all captured at 120Hz. A brief summary is shown in Tab.~\ref{tab:dataset}.
In {\bf ACCAD}, each motion sequence records a characteristic action, e.g. walking, kicking, etc., and actions are performed on the same ground plane. In {\bf CMU}, many motion sequences contain multiple actions, which are  not restricted to a common ground plane, e.g.~actions like climbing the stairs or jumping from a high place are included.  
We discard sequences shorter than 120 frames (1 second).
\begin{table}[ht]
    \scriptsize
    \centering
    \caption{Summary of our used datasets. The symbol `\#' denotes `the number of'.}
    \begin{tabular}{cccccc}
    \toprule
      dataset & \#sequences & \#subjects & \multicolumn{3}{c}{\#frames per sequence statistics} \\
      &&& average & min & max \\
    \midrule
    {\bf ACCAD} \cite{accad} & 252 & 20 &  764 & 202 & 6361 \\
    {\bf CMU} \cite{mocap_cmu} & 2061 & 106 & 1702 & 128 & 22948 \\
    \midrule
    {\bf HumanEva} \cite{sigal2010humaneva} & 28 &3  & 2180 & 360 & 3552 \\
    {\bf MPI-Mosh} \cite{Loper:SIGASIA:2014} & 77 & 19 & 1413 & 362 & 3287\\
    \bottomrule
    \end{tabular}
    \label{tab:dataset}
\end{table}

\subsection{Investigated Models}
\label{sec:baseline_method}


\paragraph{Baseline 1: QuaterNet \cite{pavllo2019modeling}.}
From related motion generation methods we select the QuaterNet as our baseline, since it not only achieves state-of-the-art motion prediction results, but also performs well on generating long-term motions with global movements. However, the original QuaterNet takes a motion path as input rather than generating the global trajectory. Furthermore it only considers locomotion (e.g. walking or running) in long-term generation. To use the QuaterNet for our task, and enable a fair comparison, we make the following modifications to \cite{pavllo2019modeling}: (1) generating global body translations and local body poses jointly, (2) replacing the quaternion representation by the 6D rotation representation, due to its better performance \cite{zhou2019continuity} and fair comparison with ours, (3) adding a sampling layer to produce non-deterministic motion sequences as ours, (4) adding $\alpha$-residual to overcome the first-frame jump. We denote the version without random sampling as ``Q'' model, and the version with random sampling as ``VQ'' model, which are trained with the same loss as ours (see Sec. \ref{sec:loss}).


\paragraph{Baseline 2: STCN \cite{aksan2018stcn}.}
From the related variational Bayesian methods, we select the stochastic temporal convolutional network (STCN) as our baseline, due to its superior performance on handwriting and speech generation. Also, comparing to another CNN-based method proposed in \cite{yan2019convolutional}, STCN is auto-regressive for generating arbitrary long sequences, and has a data-adaptive latent sequence prior. To use STCN for our task and enable a fair comparison of the network architecture, our primal modifications are: (1) increasing the receptive field in the temporal encoder to 128, and (2) using the reconstruction loss in Sec. \ref{sec:loss}. We retain its original KL-divergence term, and also learn a latent sequence prior from data. We denote our modified STCN as ``S'' model.


\paragraph{Our model instances.}
We train our cross-conditional VRNN models either on {\bf ACCAD} or {\bf CMU}, for analyzing the influence of training data. Additionally, we train models either with or without the $\alpha$-residual connection, for discovering its impact in addition to overcoming first-frame jump. Moreover, we select either LSTM \cite{hochreiter1997long} or GRU \cite{gru} as our RNN cells, for studying the influence of RNN cell types. For short, we denote the family of our model instances as ``{C}'' models.

\subsection{Evaluation on Model Representation Power}

We input every testing sequence into the models, and compute the reconstruction error ($e_{rec}$), the time difference reconstruction error ($e_{trec}$), and the negative ELBO ($-\log P$). Note that $-\log P$ includes the loss weights, but does not include the VPoser loss.

Results are shown in Tab. \ref{tab:repr_power}. 
First, we can see the $\alpha$-residual considerably reduces $e_{rec}$, which is probably caused by the weighted-average between the ground truth input and the network output. The temporal difference reconstruction is improved as well, which indicates that $\alpha$-residual can improve the model representation power consistently.
Second, C-models trained on {\bf CMU} consistently outperform their counterparts trained on {\bf ACCAD}, indicating that a larger training set with more motion variation is favourable. In addition, we observe that the C-models with GRU perform slightly better than their counterparts with LSTM. Moreover, the S-model has larger time difference reconstruction errors, which suggests that the reconstructed motion is less smooth.


\begin{table}[h!]
    \centering
    \scriptsize
    \caption{Model representation performance on test sets. The model name is specified in terms of \{model type\}-\{with $\alpha$-residual\}-\{RNN cell\}-\{training data\}. The best results are highlighted in boldface. Note that the $-\log P$ of the S-model is not comparable with others due to different formulations.}
    \begin{tabular}{lcccccc}
      \toprule
         & \multicolumn{3}{c}{\bf HumanEva} & \multicolumn{3}{c}{\bf MPI-Mosh} \\ 
         Model & $e_{rec}$ & $e_{trec}$ & $-\log P$ & $e_{rec}$ & $e_{trec}$ & $-\log P$\\
         \midrule
         Q-{\bf ACCAD}                   & 0.013          &0.002             & -             & 0.013         & 0.001     & -\\ 
         VQ-{\bf ACCAD}                  & 0.301             & 0.005             & 0.325         & 0.171         & 0.004     & 0.190 \\
         VQ-$\alpha$Res-{\bf ACCAD}      & 0.023    & 0.002    & 0.047         & {0.015} & 0.001 & 0.034 \\
         VQ-{\bf CMU}                    & 0.302             & 0.005             & 0.326         & 0.173 & 0.004 & 0.192 \\
         VQ-$\alpha$Res-{\bf CMU}        & 0.031             & 0.002    & 0.041 & 0.019 & 0.001 & 0.024 \\
         \midrule
         S-{\bf ACCAD} & 0.091 & 0.005 & 0.130* & 0.088 & 0.005 & 0.126*\\
         S-{\bf CMU} & 0.063 & 0.003 & 0.108* & 0.056 & 0.003 & 0.099* \\
         \midrule
         C-LSTM-{\bf ACCAD}             & 0.256 & 0.005             & 0.283             & 0.155         & 0.004 & 0.178 \\
         C-$\alpha$Res-LSTM-{\bf ACCAD} & 0.040 & 0.002    & 0.053             & 0.017         & 0.001 & 0.024 \\
         C-$\alpha$Res-GRU-{\bf ACCAD}  & 0.031 & 0.002    & 0.043             & 0.016         & 0.001 & 0.023 \\
         C-LSTM-{\bf CMU}               & 0.071 & 0.003             & 0.091             & 0.060         & 0.003 & 0.078 \\
         C-$\alpha$Res-LSTM-{\bf CMU}   & 0.010 & 0.002    & 0.022             & 0.008         & 0.001 & 0.015 \\
         C-$\alpha$Res-GRU-{\bf CMU}    & {\bf 0.009} & \textbf{0.002} & \textbf{0.021} & \textbf{0.007} & \textbf{0.001} & \textbf{0.015} \\
         \bottomrule
    \end{tabular}
    \label{tab:repr_power}
\end{table}

\subsection{Evaluation on Motion Generation}
For each sequence in {\bf HumanEva} and {\bf MPI-Mosh}, we separately input the first single frame, the first 10\% of frames, and the first 50\% of frames into the model, and generate the rest of the sequence. 
We only compare models with $\alpha$-residual, since models without $\alpha$-residual lead to obvious first-frame jumps.
Moreover, we find the Q-model without random sampling, and the S-model trained on {\bf CMU} produce unrealistic body poses very quickly, and hence do not compare them with others. This indicates that a model good at representing sequences is perhaps bad at generating sequences.

\subsubsection{Motion Frequency}
The motion generation result is non-deterministic, and it can happen that generated results are plausible but different from the ground truth. Following \cite{Hernandez_2019_ICCV}, we compare the frequency power spectrum distribution between generated results and the ground truth. 
We apply the fast Fourier transform over time, and then compute the power spectrum distribution for each feature dimension. 
As in \cite{Hernandez_2019_ICCV}, we propose to use two metrics: (1) Power spectrum entropy ratio (PSER), i.e. the entropy increasing ratio of the generated result over the ground truth. The closer to 0, the better. A positive value indicates noise, and a negative value indicates lack of variations. (2) The power spectrum KL-divergence (PSKLD), where lower is better.

We average the scores for all sequences as the final results, which are presented in Tab. \ref{tab:power_spectrum}. We can see that our proposed C-models outperform the baselines consistently. The generated motion from the S-model and the VQ-model has considerably lower frequencies than ground truth, indicating the generated motion lacks variation. Between the C-models, training on {\bf CMU}, which is larger and has more motion variations than {\bf ACCAD}, can lead to high frequency in the generated motion. Also, we find that LSTM is superior to GRU when trained by {\bf CMU}, but inferior when trained by {\bf ACCAD}. 
Moreover, we note that comparing performance between columns in Tab. \ref{tab:power_spectrum} is not appropriate, because the sequence length can heavily influence the precision of frequency computed by FFT.

\begin{table}[h!]
    \centering
    \scriptsize
    \caption{Evaluation results on motion frequency. Scores are shown in terms of {\bf PSER/PSKLD}, in which the {\bf PSKLD} scores are in the unit of $10^{-3}$. The best results are in boldface.   }
    \begin{tabular}{lcccccc}
      \toprule
         & \multicolumn{3}{c}{\bf HumanEva} & \multicolumn{3}{c}{\bf MPI-Mosh} \\ 
         Model & 1-frame & 10\%-frame & 50\%-frame & 1-frame & 10\%-frame & 50\%-frame\\
         \midrule
         VQ-$\alpha$Res-{\bf ACCAD} & -0.57/0.91 & -0.53/1.00 & -0.51/1.70 & -0.61/1.37 & -0.59/1.50 & -0.54/2.31 \\
         VQ-$\alpha$Res-{\bf CMU} & -0.74/0.95 &-0.72/1.01  & -0.63/1.62 & -0.78/1.60 & -0.70/1.63 & -0.59/2.53 \\
         \midrule
         S-{\bf ACCAD} & -0.72/0.91 & -0.72/0.96 & -0.67/1.50 & -0.87/1.98 & -0.76/1.82 & -0.69/2.83 \\
                 \midrule
         C-$\alpha$Res-LSTM-{\bf ACCAD}& -0.44/\textbf{0.82} & -0.36/0.89 &-0.41/1.53  &-0.68/1.44  & -0.54/1.48 & -0.42/2.34  \\
         C-$\alpha$Res-GRU-{\bf ACCAD}& -0.31/\textbf{0.82} & -0.30/\textbf{0.86} & -0.29/\textbf{1.45} & -0.34/1.24 & -0.33/1.38 & -0.36/2.28 \\
         C-$\alpha$Res-LSTM-{\bf CMU} & \textbf{0.05}/0.92 & \textbf{0.01}/1.00 & \textbf{-0.07}/1.64 & \textbf{-0.11}/1.25 & \textbf{-0.13}/1.36 & \textbf{-0.14}/2.17 \\
         C-$\alpha$Res-GRU-{\bf CMU} & -0.16/0.88 & -0.16/0.91 & -0.22/1.48 & -0.28/\textbf{1.23} & -0.29/\textbf{1.33} & -0.28/\textbf{2.10} \\
         \bottomrule
    \end{tabular}
    \label{tab:power_spectrum}
\end{table}

\subsubsection{Diversity}
For each model and each seed input sequence, we generate three sequences. The diversity is evaluated by standard deviation. Results are presented in Tab. \ref{tab:diversity}. We can see that our methods outperform the baselines. Within the C-model family, training on a larger dataset with more motion variations can improve the diversity. Additionally, we observe that the diversity results are roughly consistent with the PSER results shown in Tab. \ref{tab:power_spectrum}.

\begin{table}[h!]
    \centering
    \scriptsize
    \caption{Performance of generation diversity, which is indicated by standard deviations of three runs with the identical initial condition. A higher value of std indicates a better performance. The best results are highlighted in boldface.}
    \begin{tabular}{lcccccc}
      \toprule
         & \multicolumn{3}{c}{\bf HumanEva} & \multicolumn{3}{c}{\bf MPI-Mosh} \\ 
         Model & 1-frame & 10\%-frame & 50\%-frame & 1-frame & 10\%-frame & 50\%-frame\\
         \midrule 
         VQ-$\alpha$Res-{\bf ACCAD} & 0.10 & 0.12 & 0.10 & 0.09 & 0.09 & 0.07 \\
         VQ-$\alpha$Res-{\bf CMU} & 0.007 & 0.008 & 0.01 & 0.008 & 0.01 & 0.02 \\
         \midrule
         S-{\bf ACCAD} & 0.03 & 0.04 & 0.05 & 0.02 & 0.04 & 0.05 \\
        \midrule
         C-$\alpha$Res-LSTM-{\bf ACCAD}& 0.09 & 0.10 & 0.08 & 0.06 & 0.10 & 0.10  \\
         C-$\alpha$Res-GRU-{\bf ACCAD}& 0.12 & 0.11 & 0.10 & 0.12 & 0.12 & 0.11 \\
         C-$\alpha$Res-LSTM-{\bf CMU} & \textbf{0.22} & \textbf{0.20} & \textbf{0.16} & \textbf{ 0.21} & \textbf{0.20} & \textbf{0.17} \\
         C-$\alpha$Res-GRU-{\bf CMU} & 0.15 & 0.15 & 0.13 & 0.17 & 0.17 & 0.15 \\
         \bottomrule
    \end{tabular}
    \label{tab:diversity}
\end{table}

\subsubsection{Naturalness}
To evaluate the naturalness of generated motions, we perform a perceptual study on Amazon Mechanical Turk. For each generated sequence, 3 workers give a score ranging from 1 (unnatural) to 5 (very natural). As a control group, the ground truth sequence is evaluated in the same manner.
We report the results in Tab. \ref{tab:user_study}. Our method outperforms the baseline methods consistently. Also, models trained on {\bf CMU} perform comparably worse than their counterparts trained on {\bf ACCAD}, indicating that the scale of the dataset and the resulting naturalness are weakly related. 

\begin{table}[h!]
    \centering
    \scriptsize
    \caption{Use study results are presented in terms of mean$\pm$standard deviation scores. The best results are highlighted in boldface.}
    \begin{tabular}{lcccccc}
      \toprule
         & \multicolumn{3}{c}{\bf HumanEva} & \multicolumn{3}{c}{\bf MPI-Mosh} \\ 
         Model & 1-frame & 10\%-frame & 50\%-frame & 1-frame & 10\%-frame & 50\%-frame\\
         \midrule
         ground truth & \multicolumn{3}{c}{\sf 3.77$\pm$1.22} & \multicolumn{3}{c}{\sf 3.73$\pm$1.20} \\
         \midrule
         VQ-$\alpha$Res-{\bf ACCAD} & 2.37$\pm$1.31      & 2.88$\pm$1.40  & 2.89$\pm$1.38 & 2.50$\pm$1.41 & 2.53$\pm$1.30 & 3.00$\pm$1.43 \\
         VQ-$\alpha$Res-{\bf CMU} & 3.05$\pm$\textbf{1.18}  & 3.16$\pm$1.34    & 2.99$\pm$1.34   & 2.64$\pm$1.25 & 3.29$\pm$1.28 & 2.94$\pm$1.36 \\
         \midrule
         S-{\bf ACCAD} & 3.05$\pm$1.19    & 2.71$\pm$1.33                     & 3.06$\pm$1.33  & 2.88$\pm$1.24 & 2.87$\pm$1.44 & 3.12$\pm$\textbf{1.27}\\ 
         \midrule
         C-$\alpha$Res-LSTM-{\bf ACCAD}& \textbf{3.47}$\pm$1.24  & \textbf{3.31}$\pm$\textbf{1.21} & 3.27$\pm$\textbf{1.29} & 3.14$\pm$1.28 & \textbf{3.31}$\pm$\textbf{1.16} & \textbf{3.17}$\pm$1.30  \\
         C-$\alpha$Res-GRU-{\bf ACCAD}& 3.44$\pm$\textbf{1.18}  & 3.18$\pm$1.36   & 3.10$\pm$1.43   & \textbf{3.20}$\pm$\textbf{1.18} & 3.23$\pm$1.32 & 3.31$\pm$1.31 \\
         C-$\alpha$Res-LSTM-{\bf CMU} & 3.05$\pm$1.33          & 2.98$\pm$1.34    & \textbf{3.35}$\pm$1.39 & 2.99$\pm$1.36 & 3.01$\pm$1.36 & 3.38$\pm$1.32 \\
         C-$\alpha$Res-GRU-{\bf CMU} & 3.16$\pm$\textbf{1.18}  & 3.20$\pm$1.32    & 3.10$\pm$1.32   & 3.18$\pm$1.25 & 3.10$\pm$1.24 & 3.13$\pm$\textbf{1.27} \\
         \bottomrule
    \end{tabular}
    \label{tab:user_study}
\end{table}
\section{Conclusion}
\label{sec:conclusion}

In this paper, we address the task of generating ``perpetual'' motions, given a static body pose in the beginning. We propose a two-stream variational RNN network, in which the changes of the global trajectory and the body pose are conditioned on each other. With a novel KL-divergence term, we incorporate temporal dependencies in the latent sequence prior, and effectively overcome posterior collapse. To verify the effectiveness and perform fair comparisons, we establish a systematical pipeline to evaluate the model representation power, motion frequency, diversity and naturalness.  

However, our method still has some limitations. For example, foot sliding still exists in some results, and we plan to introduce physical constraints, e.g. foot-ground contact friction, as a solution. Also, the model might jump into an infinite loop after long time generation, although a stochastic mechanism is already employed. A potential solution is to introduce another stochastic mechanism, which allows random sampling from the activity level. We expect that our model can pave the way for future studies, like generating motions from natural languages or environments.

\paragraph{Broader Impact.}
This work has positive impact towards understanding how a real person synthesizes motions, and hence can be related to cognitive psychology and neuroscience. In addition, modeling human motion effectively could be favourable for simulations in biomedical engineering, and patient analysis in neurology and orthopedics. Moreover, personalized motion has potentials to be a biometric measure, which could cause certain privacy issues. 

\paragraph{Acknowledgements.} We appreciate the insightful discussions with Otmar Hilliges and Emre Aksan about STCN and deep variational Bayesian methods.

\paragraph{Disclosure.} In the last five years, MJB has received research gift funds from Intel, Nvidia, Adobe, Facebook, and Amazon. While MJB is a part-time employee of Amazon, his research was performed solely at, and funded solely by, MPI. MJB has financial interests in Amazon and Meshcapade GmbH.

\bibliographystyle{unsrt} 
\bibliography{references}

\clearpage
\begingroup
\onecolumn 

\appendix
\begin{center}
\Large{\bf Perpetual Motion: \\ Generating Unbounded Human Motion \\ **Appendix**}
\end{center}

\setcounter{page}{1}


\section{More Details of Experiments}

\subsection{Details of Investigated Models}
\label{sec:app:models}
In Sec.~\ref{sec:baseline_method}, we have demonstrated how to modify two state-of-the-art methods to fit our task. Here we present more details.

\paragraph{Baseline 1: QuaterNet \cite{pavllo2018quaternet}.}
We have derived two versions from the original QuaterNet, i.e.~the Q-model and the VQ-model. Their architectures are shown in Fig.~\ref{fig:app:qmodel}. We change the quaternion rotation representation to the 6D continuous representation \cite{zhou2019continuity}, and remain the two-layer GRU cells. In our setting, the hidden variable dimension in GRU is 1000, the noise dimension in the VQ-model is 32. The first fc layer in the VQ-model decoder projects vectors from the dimension of 32 to the dimension of 1000, and the second fc layer projects vectors from 1000 to 135 (3D body translation and 132D body joint rotation). The Q-model is trained with the reconstruction loss in Eq.~\ref{eq:loss_recon}, as in our C-model family. The training loss of the VQ-model is identical to the loss of C-models, as demonstrated in Sec.~\ref{sec:loss}.

\begin{apdxfig*}[h!]
    \centering
    \includegraphics[width=\linewidth]{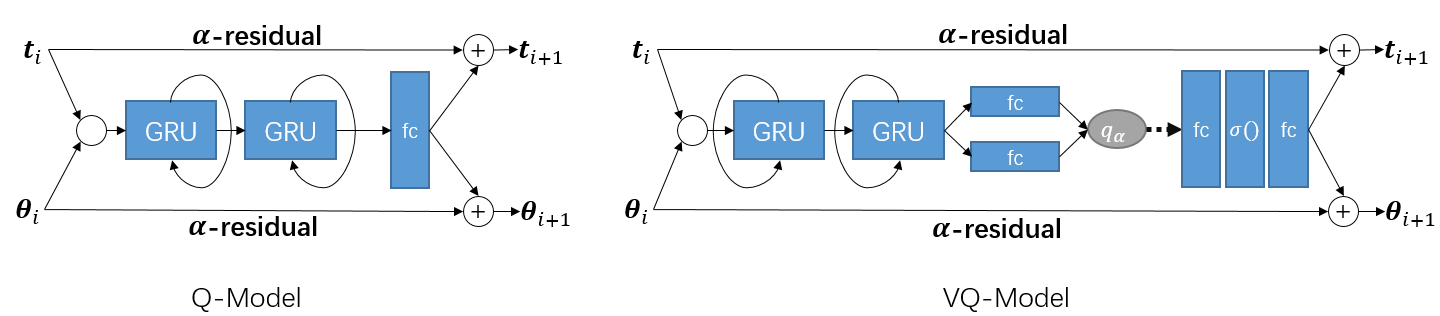}
    \caption{Illustration of our modified QuaterNet models. The notations are identical to Fig.~\ref{fig:crossmotionvar}. The activation function $\sigma(\cdot)$ is Swish \cite{ramachandran2017searching}. The $\alpha$-residual connection is optional. The two fully-connected (fc) layers after the stack of GRUs are to derive the mean and the logarithmic variance of the inference posterior $q_{\alpha}(\cdot)$, respectively.    }
    \label{fig:app:qmodel}
\end{apdxfig*}

\paragraph{Baseline 2: STCN \cite{aksan2018stcn}.}
We first concatenate the body translation and the body pose to form the motion sequence $\bm{X}_{1:i}$, which is then modeled by our modified STCN, i.e.~the S-model. The architecture of the S-model is presented in Fig.~\ref{fig:app:smodel}. As in our C-model and the Q-model, the input and output of the S-model are sub-sequences with time shift. 
To train the S-model, the reconstruction loss Eq.~\ref{eq:loss_recon} is used. In addition, the latent prior is data-adaptive, and is obtained by minimizing its KL-divergence with the inference posterior. Moreover, in our S-model the latent prior and the inference posterior take the same hidden sequence as input, which is in contrast to the original STCN. During testing, the S-model works in an auto-regressive manner, producing the next single frame based on previous 128 frames. The S-model with the original latent prior setting \cite{aksan2018stcn} is also tested in this appendix \ref{sec:app:stcn-origin}.

\begin{apdxfig*}[h!]
    \centering
    \includegraphics[width=\linewidth]{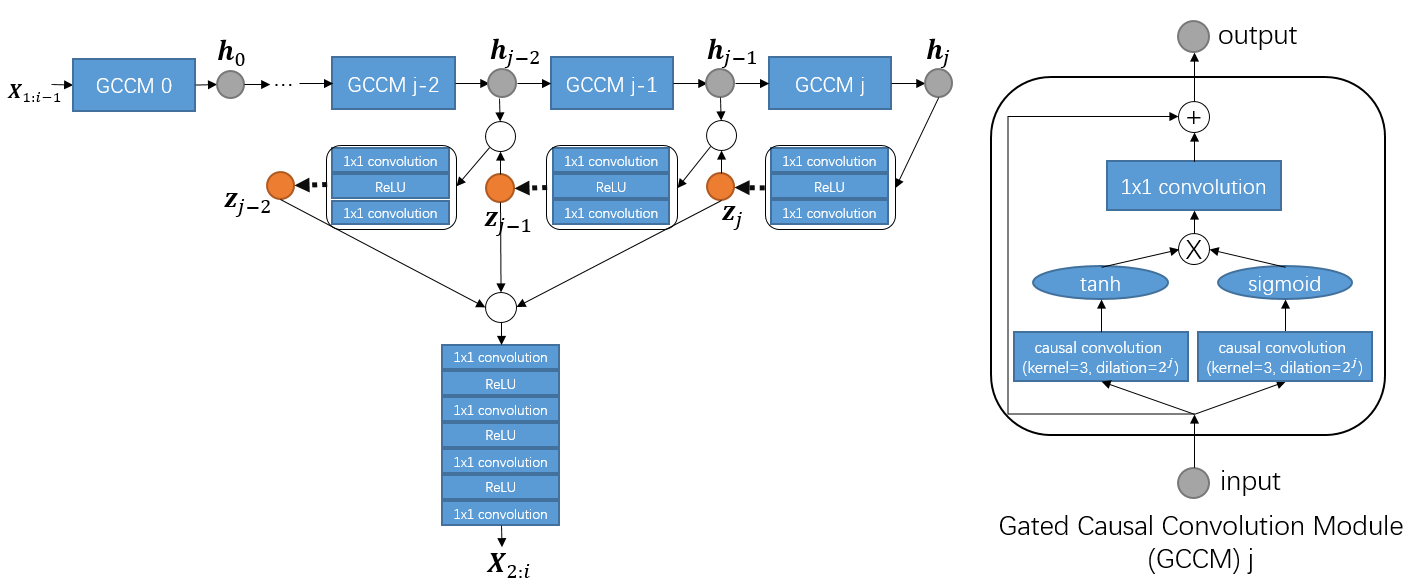}
    \caption{Illustration of our modified STCN model. The notations are identical to Fig.~\ref{fig:crossmotionvar}. In our experiments, we set $j=6$, and hence the maximal receptive field in the temporal convolution encoder is 128. The latent network has 3 layers, and uses the last 3 hidden variables in the temporal convolution encoder to derive latent variables. Then, all latent variables are concatenated to produce the prediction via the decoder.  }
    \label{fig:app:smodel}
\end{apdxfig*}

\subsection{Details of The User Study}
Our user study follows the protocol in \cite{PSI:2019}, and our user study interface is shown in Fig.~\ref{fig:app:userstudy}. Specifically, for each generated sequence, we paint the body mesh to gray if frames are given, and paint the body mesh to red if frames are generated. As a control group, we let users to evaluate the ground truth sequences as well. In each ground truth sequence, we randomly paint the body mesh to gray either in the first frame, or in the first 10\% frames, or in the first 50\% frames, and paint the body mesh to red in the remaining frames. Moreover, to avoid unknown technical problems from the user side, e.g. some users cannot play the video, we let users give a 0 score if they cannot play the video. We find the ratio of valid user study results (with a non-zero score) is 93.1\%. Invalid results are uniformly distributed across all user study groups, and we exclude them from calculating the final results in the table.

\begin{apdxfig*}[h!]
    \centering
    \includegraphics[width=\linewidth]{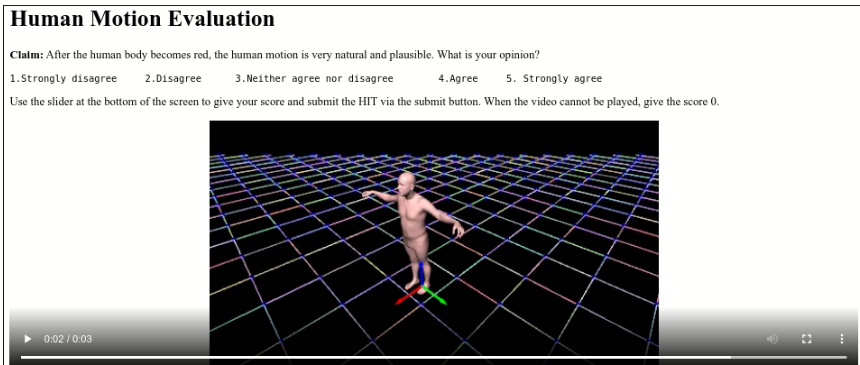}
    \caption{The user study interface in Amazon Mechanical Turk.  }
    \label{fig:app:userstudy}
\end{apdxfig*}

\section{More Analysis on Methods}

\subsection{Our Method versus Related Variational Bayesian Methods}
In this paper, we design our two-stream variational RNN model based on Eq.~\ref{eq:elbo}, in which we assume $q_{\alpha}(\bm{Z}_{2:i} | \bm{x}_i) = q_{\alpha}(\bm{Z}_{2:i})$, namely, the inference posterior (encoder) performs prediction. 
We notice that such assumption leads to large architecture differences between our method and previous variational Bayesian methods like \cite{chung2015recurrent,aksan2018stcn}. Methods without such assumption leads to an autoencoder without time shifting between input and output, and prediction is actually performed by the data-adaptive prior during the testing phase. Specifically, without such assumption, Eq.~\ref{eq:elbo} is changed to 
\begin{equation}
    \label{eq:elbo2}
    \begin{split}
    \log p({\bm x}_i | {\bm X}_{1:i-1}) &\geq \mathbb{E}_{\bm{Z}_{2:i} \sim q_{\alpha}(\bm{Z}_{2:i} |{\color {red} \bm{x}_i, {\bm X}_{1:i-1} } ) }[ \log p_{\beta}({\bm x}_i | \bm{Z}_{2:i}, {\bm X}_{1:i-1}) ]\\ &-D_{KL}\left( q_{\alpha}(\bm{Z}_{2:i} | {\color {red} \bm{x}_i, {\bm X}_{1:i-1} })\, || \, p(\bm{Z}_{2:i} | {\bm X}_{1:i-1}) \right),
    \end{split}
\end{equation}
in which the inference posterior $q_{\alpha}(\cdot)$ and the generation posterior $p_{\beta}(\cdot)$ only perform reconstruction, and prediction comes from the prior $p(\bm{Z}_{2:i} | {\bm X}_{1:i-1})$. As a consequence, formulating an explicit latent sequence prior, which is different from the inference posterior, is necessary. 

In addition, the training loss will be different from Sec.~\ref{sec:loss} as well. Provided the ground truth body translations and body poses, we feed the \textit{entire} sequences to the model as input, and obtain their reconstructed versions. Specifically, the reconstruction loss is defined as
\begin{equation}
    \begin{split}
    \mathcal{L}_{rec} &= | \hat{\bm{T}}_{1:i} - \bm{T}_{1:i} | + | \hat{\bm{\Theta}}_{1:i} - \bm{\Theta}_{1:i} | \\
    &+ \lambda_{ts} \left( | (\hat{\bm{T}}_{2:i}-\hat{\bm{T}}_{1:i-1}) - (\bm{T}_{2:i}-\bm{T}_{1:i-1}) | + | (\hat{\bm{\Theta}}_{2:i}-\hat{\bm{\Theta}}_{1:i-1}) - (\bm{\Theta}_{2:i}-\bm{\Theta}_{1:i-1}) | \right),
    \end{split}
    \label{eq:loss_rec_2}
\end{equation}
which does not have the time shift as in Eq. \ref{eq:loss_recon}. 

\paragraph{Comparison with the original STCN.}
\label{sec:app:stcn-origin}
It is noticed that our modified version of STCN \cite{aksan2018stcn} has the assumption in Eq.~\ref{eq:elbo}. In this case, the modified STCN (i.e.~the S-model in Sec.~\ref{sec:baseline_method} and Sec.~\ref{sec:app:models}) can use the same reconstruction loss with ours, and hence the comparison can focus on the network architecture. 
However, when generating motion, we find that our modified STCN trained on {\bf CMU} produces unrealistic body poses very quickly after the first frame.
A similar observation on CNN-based trajectory prediction is reported by \cite{Nikhil_2018_ECCV_Workshops}.
We conjecture that such unstable performance is caused by this assumption. Therefore, here we evaluate the performance of the original formulation of STCN, without the assumption in Eq.~\ref{eq:elbo}. Specifically, except the latent prior network, the network setting is identical to the S-model in Sec. \ref{sec:baseline_method} and Sec.~\ref{sec:app:models}. Correspondingly, the reconstruction loss becomes to Eq.~\ref{eq:loss_rec_2} rather than Eq.~\ref{eq:loss_recon}, and the ELBO becomes to Eq.~\ref{eq:elbo2} rather than Eq.~\ref{eq:elbo}. The loss weights are the same as in Sec. \ref{sec:loss}. We name this STCN version as ``SO'', meaning ``STCN Origin''. 

Unlike what we have conjectured, we find that the motion generation process based on the SO-model is not completely stable. The SO-model trained on {\bf CMU} produces unrealistic body poses quickly, e.g. after generating 300 frames. Therefore, we think such unstable behavior could be rooted at somewhere else rather than the model itself. A probable reason is, that our data pre-processing step, i.e. transforming the sequence to the AMASS coordinate, is not suitable for CNN-based models. Always starting with $(0,0,0)$ body translation could make motion generation numerically unstable. Nevertheless, our proposed RNN-based models do not encounter such instability issue, indicating that the RNN-based model is more suitable for the motion generation task.    

In the following tables \ref{tab:app:model_repr}-\ref{tab:app:userstudy}, we show the results of the SO-models with respect to the model representation power, the motion frequency, the diversity and the naturalness. Since the SO-model trained on {\bf CMU} cannot perform motion generation stably, we only evaluate its model representation power here. From the results, we can see the naturalness performance of the SO-model is comparably better than our proposed C-models, while the diversity performance is much inferior to our C-models. From the qualitative results, we observe that many generated motions are `standing still', which are plausible but lack variations.

\begin{apdxtab}[h!]
    \centering
    \caption{Performance in terms of model representation power, corresponding to Tab. \ref{tab:repr_power}.}
    \label{tab:app:model_repr}
    \begin{tabular}{lcccccc}
      \toprule
         & \multicolumn{3}{c}{\bf HumanEva} & \multicolumn{3}{c}{\bf MPI-Mosh} \\ 
         Model & $e_{rec}$ & $e_{trec}$ & $-\log P$ & $e_{rec}$ & $e_{trec}$ & $-\log P$\\
         \midrule
         SO-{\bf ACCAD} & 0.108 & 0.005 & 0.212* & 0.120 & 0.005 & 0.213*\\
         SO-{\bf CMU} & 0.078 & 0.004 & 0.271* & 0.076 & 0.004 & 0.267* \\
         \bottomrule
    \end{tabular}
    \end{apdxtab}
   
   \begin{apdxtab}[h!]
    \centering
    \caption{Performance in terms of motion frequency spectrum power, corresponding to Tab. \ref{tab:power_spectrum}.}
    \label{tab:app:frequency}
    \begin{tabular}{lcccccc}
      \toprule
         & \multicolumn{3}{c}{\bf HumanEva} & \multicolumn{3}{c}{\bf MPI-Mosh} \\ 
         Model & 1-frame & 10\%-frame & 50\%-frame & 1-frame & 10\%-frame & 50\%-frame\\
         \midrule
         SO-{\bf ACCAD} & -0.83/0.91 & -0.77/0.94 & -0.77/1.64 & -0.92/2.09 & -0.84/1.81 & -0.77/2.73 \\
         \bottomrule
    \end{tabular}
    \end{apdxtab}
   
    \begin{apdxtab}[h!]
    \centering
    \caption{Performance in terms of motion diversity, corresponding to Tab. \ref{tab:diversity}. }
    \begin{tabular}{lcccccc}
      \toprule
         & \multicolumn{3}{c}{\bf HumanEva} & \multicolumn{3}{c}{\bf MPI-Mosh} \\ 
         Model & 1-frame & 10\%-frame & 50\%-frame & 1-frame & 10\%-frame & 50\%-frame\\
         \midrule 
         SO-{\bf ACCAD} & 0.03 & 0.03 & 0.03 & 0.01 & 0.03 & 0.05 \\
         \bottomrule
    \end{tabular}
    \label{tab:app:diversity}
    \end{apdxtab}
    
    \begin{apdxtab}[h!]
    \centering
    \caption{Performance in terms of motion naturalness, corresponding to Tab. \ref{tab:user_study}.}
    \label{tab:app:userstudy}
    \begin{tabular}{lcccccc}
      \toprule
         & \multicolumn{3}{c}{\bf HumanEva} & \multicolumn{3}{c}{\bf MPI-Mosh} \\ 
         Model & 1-frame & 10\%-frame & 50\%-frame & 1-frame & 10\%-frame & 50\%-frame\\
         \midrule
         SO-{\bf ACCAD} & 3.31$\pm$1.16  & 3.32$\pm$1.19  & 3.51$\pm$1.06  & 3.28$\pm$1.07 & 3.24$\pm$1.17 & 3.32$\pm$1.18\\ 
        \bottomrule
    \end{tabular}
    \end{apdxtab}

\subsection{Influence of The $\alpha$-residual Connection}
In our paper we propose the $\alpha$-residual connection to overcome the first-frame jump artifacts, and set $\alpha=0.9$ in all trials. Here we qualitatively show its effectiveness in Fig. \ref{fig:app:ares}. Obviously, our method overcomes the first-frame jump artifacts in a very effective manner.

\begin{apdxfig*}[h!]
    \centering
    \includegraphics[width=\linewidth]{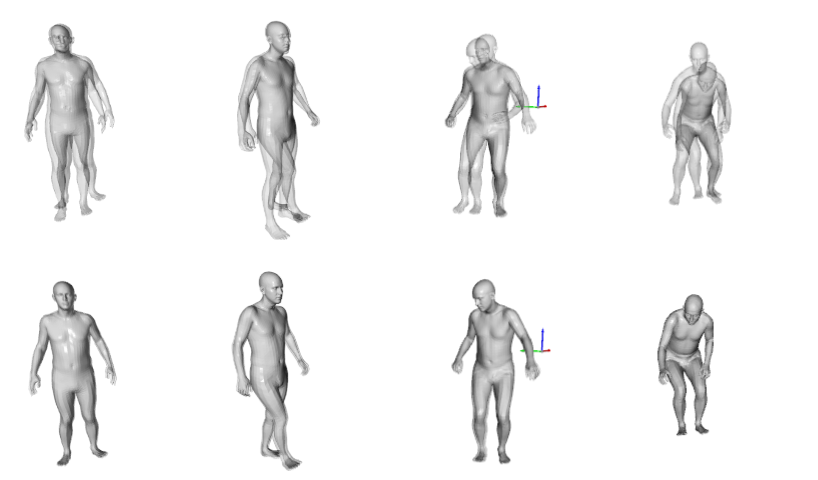}
    \caption{Illustration of the effectiveness of $\alpha$-residual to overcome first-frame jump, in which the last given frame and the first generated frame are overlaid. The first row and the second row show the results without and the with the $\alpha$-residual connection, respectively. Results within each column use the identical given frame.  }
    \label{fig:app:ares}
\end{apdxfig*}

\end{document}